\documentclass[a4paper,10pt]{article}
\usepackage[T1]{fontenc}
\usepackage{parskip}
\usepackage[font={small}]{caption}
\usepackage[margin=2.5cm]{geometry}
\usepackage{times}
\usepackage{amsmath,amssymb,amsthm}
\usepackage[affil-it]{authblk}
\usepackage{microtype}
\usepackage{listings}
\usepackage[hidelinks]{hyperref}

\date{}

\usepackage{macros}
\newif\ifarxiv

\arxivtrue

\title{Prediction performance after learning \\ in Gaussian process regression}
\author[1]{Johan W{\aa}gberg\thanks{\url{johan.wagberg@it.uu.se}}}
\author[1]{Dave Zachariah\thanks{\url{dave.zachariah@it.uu.se}}}
\author[1]{Thomas B. Sch\"on\thanks{\url{dave.zachariah@it.uu.se}}}
\author[1]{Petre Stoica\thanks{\url{petre.stoica@it.uu.se}}}
\affil[1]{Department of Information Technology, Uppsala University}

\begin{document}

\maketitle

\textbf{Please cite this version:}

Johan W{\aa}gberg, Dave Zachariah, Thomas B. Sch\"on and Petre Stoica. Prediction performance after learning in Gaussian process regression. In \textit{Proceedings of the 20th International Conference on Artificial Intelligence and Statistics (AISTATS)}, Fort Lauderdale, FL, USA, April, 2017.

\begin{center}
\begin{minipage}{.9\linewidth}
\begin{lstlisting}[breaklines,basicstyle=\small\ttfamily]
@inproceedings{waagberg2017prediction,
  author    = {W{\aa}gberg, Johan and Zachariah, Dave and Sch{\"o}n, Thomas B. and Stoica, Petre},
  title     = {Prediction performance after learning in Gaussian process regression},
  booktitle = {Proceedings of the 20th International Conference on Artificial Intelligence and Statistics (AISTATS)},
  year      = {2017},
  month     = {4},
  address   = {Fort Lauderdale, FL, USA}
}
\end{lstlisting}
\end{minipage}
\end{center}
\vspace{5em}


\begin{abstract}
\noindent
This paper considers the quantification of the prediction performance in Gaussian process regression.
The standard approach is to base the prediction error bars on the theoretical predictive variance, which is a lower bound on the mean square-error (MSE).
This approach, however, does not take into account that the statistical model is learned from the data.
We show that this omission leads to a systematic underestimation of the prediction errors.
Starting from a generalization of the Cram\'er-Rao bound, we derive a more accurate MSE bound which provides a measure of uncertainty for prediction of Gaussian processes.
The improved bound is easily computed and we illustrate it using synthetic and real data examples.
\end{abstract}

\clearpage


\section{Introduction}
In this paper we consider the problem of learning a function $ f(\x)$ from a dataset $\mathcal{D}_{N}=\{\x_i, y_i \}^N_{i=1}$ where
\begin{equation}\label{eq:standardmodel}
	y = f(\x) + \varepsilon \; \in \; \mathbb{R}.
\end{equation}
The aim is to predict $f(\x_{\star})$ at a test point $\x_{\star}$.
In machine learning, spatial statistics and statistical signal processing, it is common to model $f(\x)$ as a Gaussian process (GP) and $\varepsilon$ as an uncorrelated zero-mean Gaussian noise \parencite{Stein1999_interpolation,Bishop2006_pattern,Murphy2012_machine,PerezCruzEtAl2013_gaussian}.
This probabilistic framework shares several properties with kernel and spline-based methods \parencite{Scholkopf&Smola2002_kernels,SuykensEtAl2002_lssvm,Rasmussen&Williams2006_gaussian}.
One of the strengths of the GP framework is that both a predictor $\what{f}(\x_{\star})$ and its error bars are readily obtained using the mean and variance of $f(\x_{\star})$.
This quantification of the prediction uncertainty is valuable in itself but also in applications that involve decision making, e.g. in the exploration-exploitation phase of active learning and control \parencite{Likar&Kocijan2007_predictive,Deisenroth&Rasmussen2011_pilco,deisenroth2015gaussian}.
Another recent example is Bayesian optimization techniques using Gaussian processes \parencite{ShahriariEtAl2016_gpoptimization}.

In general, however, the model for $f(\x)$ is not fully specified but contains unknown hyperparameters, denoted $\hp$, that can be learned from data.
Plugging an estimate $\what{\hp}$ into the predictor $\what{f}(\x_\star)$ will therefore inflate its errors due to the uncertainty of the learned model itself.
In this case the standard error bounds will systematically underestimate the actual prediction errors.
One possibility is to assign a prior distribution to $\hp$ and marginalize out the
parameters from the posterior distribution of $\check{f}_\star$ \parencite{WilliamsRasmussen1996}.
While conceptually straight-forward, this approach is challenging to implement in general as it requires the user to choose a reasonable prior distribution and computationally demanding numerical integration techniques.

Our contribution in this paper is the derivation of more accurate error bound for prediction after learning, using a generalization of the Cram\'er-Rao Bound (CRB) \parencite{Rao1945_information,Cramer1946_contribution}.
The bound is computationally inexpensive to implement using standard tools in the GP framework.
We illustrate the bound using both synthetic and real data.

\section{Problem formulation and related work} 
We consider a general input space $\x \in \mathcal{X}$.
To establish the notation ahead we write the Gaussian process and the vector of hyperparameters as 
\begin{equation}\label{eq:gpmodel}
	f(\x) \sim \gp{m_\alpha(\x)}{k_\beta(\x,\x^\prime)} \;\;\; \text{and} \;\;\; \hp = \bbm \hpmean \\ \hpcov \\ \hpstd^2 \ebm,
\end{equation}
where $\hpstd^2$ denotes the variance of $\varepsilon$ in \eqref{eq:standardmodel}.
The vectors $\hpmean$ and $\hpcov$ parameterize the mean and covariance functions, $m_\alpha(\x)$ and $k_\beta(\x,\x^\prime)$, respectively.
For an arbitrary test point $\x_\star$ we write $f_\star = f(\x_\star)$ and consider the mean-square error
\begin{equation*}
	\MSE{\what{f}_\star} \triangleq \Expb{}{| f_\star - \what{f}_\star |^2},
\end{equation*}
where the expectation is taken with respect to $f_\star$ and the data $\y$.
When $\hp$ is given, the optimal predictor is 
\begin{equation}\label{eq:predictor}
	\check{f}_\star(\hp) = m_\star  + \linpred^\Transp(\y - \m),
\end{equation}
where $\linpred = \left( \mbf{K} + \hpstd^2 \mbf{I} \right)^{-1} \mbf{k}_\star$, $m_\star = m_\alpha(\x_\star)$ and $\m = [m(\x_1) \: \cdots m(\x_N)]^\Transp$.
In addition, $\mbf{k}_\star = [k_\beta(\x_\star, \x_1) \: \cdots \: k_\beta(\x_\star,\x_N) ]^\Transp$ and $\mbf{K}= \{  k_\beta(\x_i, \x_j) \}_{i,j}$.
Eq. \eqref{eq:predictor} is equal to the mean of the predictive distribution $p(f_\star|\y,\hp)$ and is a function of both $\y$ and $\hp$ \parencite{Rasmussen&Williams2006_gaussian}.
The minimum MSE then follows directly from the predictive variance, denoted $\sigma^2_{\star|y}(\hp)$.
Here, however, we provide an alternative derivation based on a generalization of the CRB \parencite{vanTrees2013_detection,Gill1&Levit1995_applications,Zachariah&Stoica2015_cramer}.
This tool will also enable us to tackle the general problem considered later on.
\setcounter{resultcnt}{0}
\begin{thm}
	When $\hp$ is known,
	\begin{equation}\label{eq:BCRB}
		\MSE{\what{f}_{\star}} \geq \underbrace{k_{\star
                    \star} - \mbf{k}^\Transp_\star \left( \mbf{K} +
                    \hpstd^2 \mbf{I} \right)^{-1}
                  \mbf{k}_\star}_{=\sigma^2_{\star|y}(\hp)},
	\end{equation}
	where $k_{\star \star} = k_\beta(\x_\star, \x_\star)$.
\end{thm}
\begin{proof}
The Bayesian Cramér-Rao Bound (BCRB) is given by
\begin{equation*}
	\MSE{\what{f}_{\star}} \geq J^{-1}_\star,
\end{equation*}
where $J_\star \triangleq\Exp{\left( \frac{\partial}{\partial f_\star}\ln p(\y, f_\star \mid \hp) \right)^2}$ is the Bayesian information of $f_\star$ \parencite{vanTrees2013_detection}.
Using the chain rule, $\ln p(\y, f_\star \mid \hp) = \ln p( f_\star \mid \y,\hp)  + \ln p(\y \mid \hp) $, we obtain
\begin{equation}\label{eq:df}
	\begin{aligned}
		\frac{\partial}{\partial f_\star} \ln p(\y, f_\star | \hp)
			&= \frac{\partial}{\partial f_\star} \ln p( f_\star | \y, \hp)  + 0 \\
			&= \frac{\partial}{\partial f_\star} \left( -\frac{1}{2}\ln (2\pi \sigma^2_{\star|y}) - \frac{1}{2\sigma^2_{\star|y}} ( f_\star - \check{f}_\star(\hp))^2 \right) \\
			&= -\sigma^{-2}_{\star|y} \left( f_\star - \check{f}_\star(\hp)\right),
	\end{aligned}
\end{equation}
under the assumptions made.
Then the Bayesian information equals
\begin{equation}
	\begin{split}
		J_\star &=\E\left[ \sigma^{-4}_{\star|y} \left( f_\star - \check{f}_\star(\hp)\right)^2 \right] = \frac{1}{\sigma^2_{\star|y}}.
	\end{split}
	\label{eq:HIM_star}
\end{equation}
\end{proof}
\textbf{Remarks:}
The lower bound \eqref{eq:BCRB} on the MSE, and the corresponding minimum error bars for a predictor $\what{f}_\star$, reflects the uncertainty of $f_\star$ alone.
The bound is attained when $\what{f}_\star$ coincides with \eqref{eq:predictor} which depends on $\hp$.
In general, however, $\hp$ is unknown and typically learned from the data.
Then the bound \eqref{eq:BCRB} will not reflect the additional errors of $\what{f}_\star$ arising from the unknown model parameters $\hp$.
The effect is a systematic underestimation of the prediction errors.
For illustrative purposes we present an example with one-dimensional inputs, see Example~\ref{ex:basicexample} below.


\begin{exmp}\label{ex:basicexample}
Consider the Gaussian process \eqref{eq:standardmodel} for $x \in \mathbb{R}$ with a linear mean function and a squared-exponential covariance function.
That is, $m_\alpha(x) = \alpha x$ and $k_\beta(x,x^\prime) = \beta^2_0 \exp\left( -\frac{1}{2\beta^2_1} \| x - x^\prime \|^2 \right)$ in \eqref{eq:gpmodel}.
The process is sampled at $N=10$ different points and the unknown hyperparameters are learned from the dataset by maximizing the marginal likelihood, $\what{\hp} = \argmax_{\hp} \; \int p(\y, \f | \hp) d\f$, where the vector $\f$ contains the $N$ latent function values in the data.

Figure~\ref{fig:basicexample} illustrates a realization of $f(x)$ along with the predicted values $\what{f}(x)$.
The error bars are obtained from \eqref{eq:BCRB} which was derived under the assumption of $\hp$ being known.
The bars severely underestimate the uncertainty of the predictor since they remain nearly constant along the input space and do not contain the example realization of $f(x)$.
\begin{figure}[!ht]
	\centering%
	\pgfplotsset{
		every axis/.append style={
			legend style={
				font=\small,
				row sep=-0pt,
				draw=none},
			axis line on top
			},
		every axis post/.append style={
			legend style={
				draw=none,
				fill=none},
			ymax=15,
			},
		every axis plot post/.append style={
			very thick,
			every mark/.append style={scale=2}},
		every tick label/.append style={
			font=\footnotesize
			}
	}
    \setlength{\figureheight}{.40\columnwidth}
    \setlength{\figurewidth}{.80\columnwidth}
    
  %
	\ifarxiv
		\includegraphics{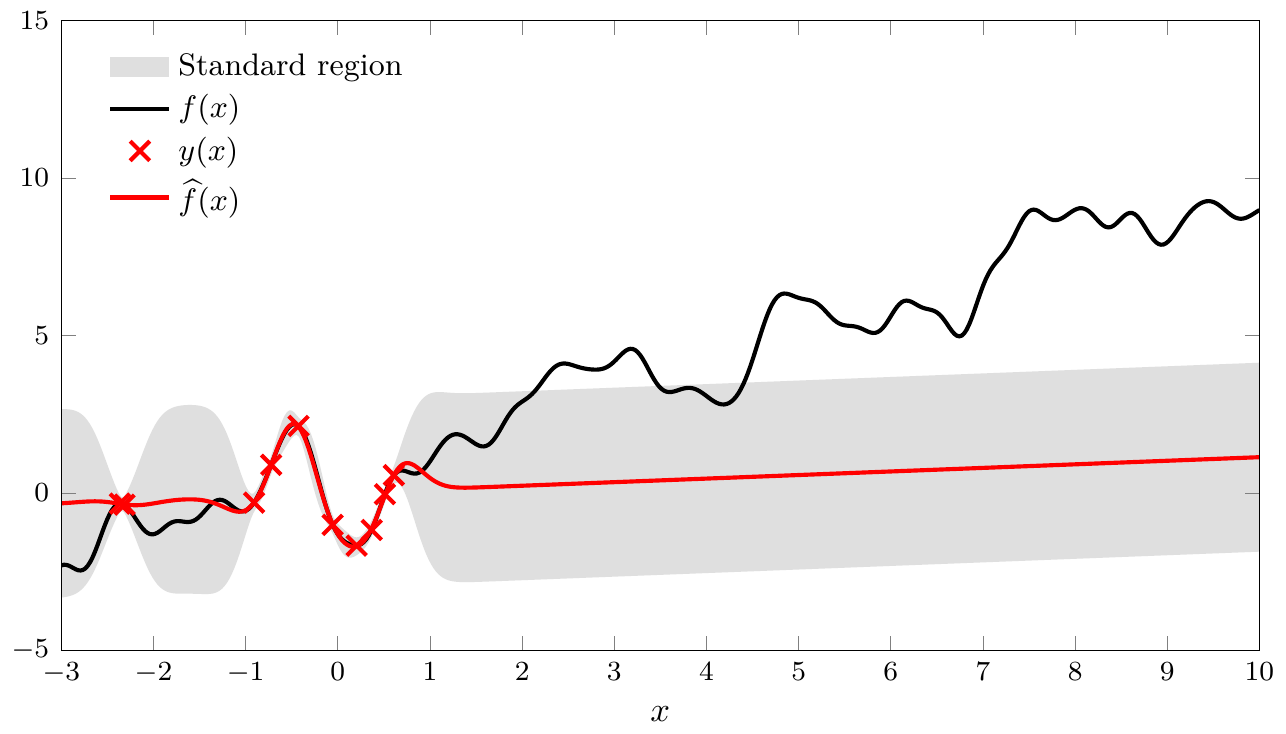}
	\else
		\tikzsetnextfilename{linear_mean_sq_exp_cov_part1}
		\input{fig/linear_mean_sq_exp_cov_part1.pgf}
	\fi

	\caption{
		Predictions of $f(x)$ using hyperparameters that have been learned from data.
		The shaded error bars are credibility regions corresponding to $\what{f}(x)\pm 3 \sigma_{\star|y}$.
	}
	\label{fig:basicexample}
\end{figure}
\end{exmp}

A tighter MSE bound than \eqref{eq:BCRB} has been derived for the special case in which $\hpcov$ is assumed to be known and the mean function is linear in the parameters, i.e., $m_\alpha(\x) = \hpmean^\Transp \mbf{u}(\x)$ where $\mbf{u}(\x)$ is a given basis function \parencite{Stein1999_interpolation}.
In the statistics literature, there have been attempts to extend to the analysis to models in which the covariance parameters $\hpcov$ are unknown.
The results do, however, not generalize to nonlinear $m_\alpha(\x)$ and are either based on computationally demanding Taylor-series expansions or bootstrap techniques \parencite{Zimmerman&Cressie1992_mean, DenHertogEtAl2006_correct}.

The goal of this paper is to derive a computationally inexpensive lower bound on the MSE that will provide more accurate error bars on $\what{f}_\star$ when $\hp$ is unknown.

\section{Prediction errors after learning the hyperparameters}
In the general setting when $\hp$ is unknown we have the following lower bound on the MSE. 
\begin{thm}\label{res:hcrb}
	When $\hp$ is learned from $\y$ using an unbiased estimator, we have that:
	\begin{equation}\label{eq:HCRB}
		\boxed{\text{MSE}(\what{f}_\star) \geq \sigma^2_{\star|y} + \dmean^\Transp\mbf{M}^{-1} \dmean,}
	\end{equation}
	where 
	\begin{equation}
		\begin{aligned}
			\dmean = \frac{\partial}{\partial \hpmean} (m_\star - \m^\Transp \linpred ) \quad \text{and} \quad
			\mbf{M} = \frac{\partial \m^\Transp}{\partial \hpmean} (\K + \sigma^2 \mbf{I} )^{-1} \frac{\partial \m}{\partial \hpmean^\Transp} .
		\end{aligned}
	\end{equation}
	Comparing with \eqref{eq:BCRB}, the non-negative term
        $\dmean^\Transp \mbf{M}^{-1} \dmean \geq 0$ is the additional
        error incurred due to the lack of information about $\hp$. 
\end{thm}
\textbf{Remarks:}
Eq.~\eqref{eq:HCRB} is the Hybrid Cram\'er-Rao Bound which we abbreviate as $\HCRB(\hp) \triangleq \sigma^2_{\star|y} + \dmean^\Transp\mbf{M}^{-1} \dmean$ \parencite{Rockah&Schultheiss1987_array,vanTrees2013_detection}.

First, note that $\dmean^\Transp\mbf{M}^{-1} \dmean$ will be non-zero even in the simplest models where the data has an unknown constant mean, i.e., $m_\alpha(\x) \equiv \alpha$.

Second, \eqref{eq:HCRB} depends on the unknown covariance parameters $\hpcov$ only via $\mbf{M}$ and not through any gradients as would be expected.
As we show in the proofs below, this is follows from the properties of the Gaussian data distribution.
In the special case of linear mean functions, $m_\alpha(\x) = \hpmean^\Transp \mbf{u}(\x)$, \eqref{eq:HCRB} coincides with MSE of the universal kriging estimator which assumes $\hpcov$ to be known \parencite{Stein1999_interpolation}.

Third, under standard regularity conditions, the maximum likelihood approach will yield estimates of $\hp$ that are asymptotically unbiased and attain their corresponding error bounds \parencite{vanTrees2013_detection}.

\begin{proof}
	The HCRB for $f_\star$ is given by
	\begin{equation}\label{eq:HCRB_general}
		\begin{split}
		 \MSE{\what{f}_\star} \geq (J_\star - \FIM^\Transp_{\theta,\star} \FIM^{-1}_{\theta} \FIM_{\theta,\star} )^{-1}, 
		\end{split}
	\end{equation}
	where the matrices are given by the hybrid information matrix
	\begin{equation}\label{eq:HIM}
		\begin{aligned}
		\FIM
			\triangleq \Exp{
			\bbm
				\frac{\partial \ln p(\y, f_\star | \hp )}{\partial f_\star} \\
				\frac{\partial \ln p(\y, f_\star | \hp )}{\partial \hp}
			\ebm
			\bbm
				\frac{\partial \ln p(\y, f_\star | \hp )}{\partial f_\star} \\
				\frac{\partial \ln p(\y, f_\star | \hp )}{\partial \hp}
			\ebm^\Transp}
			= 
			\bbm
				J_\star & \FIM^\Transp_{\theta,\star} \\ \FIM_{\theta,\star} & \FIM_\theta
			\ebm.
		\end{aligned}
	\end{equation}
	To prepare for the subsequent steps, we introduce $\yf = [\y^\Transp \; f_\star]^\Transp$ and let $\yfmean$ and $\yfcov$ denote the joint mean and covariance matrix respectively, i.e. $\yf \sim \mathcal{N}(\yfmean, \yfcov)$.
	Next, we define the linear combiner
	\begin{equation*}
		\wtilde{\linpred}^\Transp = \bbm \0 & 1\ebm - \linpred^\Transp \bbm \mbf{I} & \0 \ebm
	\end{equation*}
	and note that 
	\begin{equation}\label{eq:linearcombiner}
		\wtilde{\linpred}^\Transp(\yf - \yfmean) =
		\wtilde{\linpred}^\Transp \left( \bbm \y \\ f_\star \ebm - \bbm \m \\ m_\star \ebm \right)= f_\star - \what{f}_\star.
	\end{equation}
	Similarly, $\wtilde{\linpred}^\Transp \frac{\partial \yfmean}{\partial \hpmean^\Transp} = \frac{\partial }{\partial \hpmean^\Transp} \tilde{\linpred}^\Transp \yfmean = \dmean^\Transp$.

	To compute the block $\FIM_{\theta,\star}$ in \eqref{eq:HIM}, we first establish the following derivatives:
	\begin{equation*}
		\begin{aligned}
			\frac{\partial}{\partial \hpmean} \ln p(\yf|\hp)
				&=\frac{\partial \yfmean^\Transp}{\partial \hpmean} \yfcov^{-1}(\yf - \yfmean), \\
			\frac{\partial}{\partial \beta_i} \ln p(\yf|\hp)
				&= -\frac{1}{2} \text{tr}\left\{ \yfcov^{-1} \frac{\partial \yfcov}{\partial \beta_i}\right\} + \frac{1}{2} ( \yf - \yfmean )^\Transp \yfcov^{-1} \frac{\partial \yfcov}{\partial \beta_i} \yfcov^{-1} ( \yf - \yfmean ), \\
			\frac{\partial}{\partial \hpstd^2} \ln p(\yf|\hp)
				&= -\frac{1}{2}\text{tr}\left\{ \yfcov^{-1} \frac{\partial \yfcov}{\partial \hpstd^2}\right\} + \frac{1}{2} ( \yf - \yfmean )^\Transp \yfcov^{-1}\frac{\partial \yfcov}{\partial \hpstd^2} \yfcov^{-1} ( \yf - \yfmean )\\
			\frac{\partial}{\partial f_\star} \ln p(\yf|\hp)
				&= - \hpstd^{-2}_{\star|y}\wtilde{\linpred}^\Transp( \yf - \yfmean),
		\end{aligned}
	\end{equation*}
	where the last equality follows from \eqref{eq:df} and \eqref{eq:linearcombiner}.
	Then we obtain
	\begin{align*}
		\Exp{ \frac{\partial}{\partial f_\star} \ln p(\yf|\hp) \frac{\partial}{\partial \hpmean^\Transp} \ln p(\yf|\hp)}
			&= - \hpstd^{-2}_{\star|y}\wtilde{\linpred}^\Transp \Exp{( \yf - \yfmean) (\yf - \yfmean)^\Transp}\yfcov^{-1} \frac{\partial \yfmean}{\partial \hpmean^\Transp} \\
			&=- \hpstd^{-2}_{\star|y} \dmean^\Transp
	\end{align*}
	and
	\begin{align*}
		\Exp{\frac{\partial}{\partial f_\star} \ln p(\yf|\hp)\frac{\partial}{\partial \beta_i} \ln p(\yf|\hp)}
			&= \frac{1}{2} \hpstd^{-2}_{\star|y} \wtilde{\linpred}^\Transp \underbrace{\E[(\yf - \yfmean)]}_{=\0} \text{tr}\left\{\yfcov^{-1} \frac{\partial \yfcov}{\partial \beta_i} \right\} -\\
			&\quad- \frac{1}{2} \hpstd^{-2}_{\star|y} \wtilde{\linpred}^\Transp\underbrace{ \Exp{ ( \yf - \yfmean)  ( \yf - \yfmean )^\Transp \yfcov^{-1} \frac{\partial \yfcov}{\partial \beta_i} \yfcov^{-1} ( \yf - \yfmean) }}_{=\0} \\
			&= 0.
	\end{align*}
	Similarly, $\E \left[ \frac{\partial}{\partial f_\star} \ln p(\yf|\hp)\frac{\partial}{\partial \hpstd^2} \ln p(\yf|\hp) \right] = 0$. 	Therefore
	\begin{equation}\label{eq:HIM_cross}
		\begin{split}
			\FIM^\Transp_{\theta,\star} = \bbm - \hpstd^{-2}_{\star|y} \dmean^\Transp & \0 & 0 \ebm. 
		\end{split} 
	\end{equation}
	Next, using the distribution of $\yf$, $\FIM_\theta$ is obtained via Slepian-Bangs formula \parencite{Slepian1954_estimation,Bangs1971_array,Stoica&Moses2005_spectral}:
	\begin{equation*}
		\begin{split}
			\{ \FIM_\theta \}_{i,j} &= \frac{\partial \yfmean^\Transp}{\partial\theta_i} \yfcov^{-1} \frac{\partial \yfmean}{\partial \theta_j} + \frac{1}{2} \text{tr} \left\{ \yfcov^{-1} \frac{\partial \yfcov}{\partial \theta_i}  \yfcov^{-1}  \frac{\partial \yfcov}{\partial \theta_j}  \right \}.
		\end{split}
	\end{equation*}
	This yields a block-diagonal matrix
	\begin{equation}\label{eq:HIM_theta}
		\begin{split}
		\FIM_\theta &= \bbm \frac{\partial \yfmean^\Transp}{\partial \hpmean } \yfcov^{-1} \frac{\partial \yfmean}{\partial \hpmean^\Transp } & \0  & \0 \\ \0 & * & * \\ \0 & * & * \ebm.
		\end{split}
	\end{equation}
	where the right-lower block does not affect \eqref{eq:HCRB_general} due to the zeros in \eqref{eq:HIM_cross}.
	Inserting \eqref{eq:HIM_cross}, \eqref{eq:HIM_theta} and \eqref{eq:HIM_star} into \eqref{eq:HCRB_general} then yields
	\begin{equation}
		\begin{aligned}
			 \mse \! \left({\what{f}_\star}\right)
				&\!\geq\! \left(\sigma^{-2}_{\star|y} \!-\! \hpstd^{-2}_{\star|y}\dmean^\Transp \left(\frac{\partial \yfmean^\Transp}{\partial\hpmean}\yfcov^{-1} \frac{\partial \yfmean}{\partial\hpmean^\Transp } \right)^{\mathrlap{-1}} \dmean\hpstd^{-2}_{\star|y} \right)^{\!\!\!\!-1} \\
				&= \sigma^{2}_{\star|y} +  \dmean^\Transp \left(   \frac{\partial \yfmean^\Transp}{\partial\hpmean} \yfcov^{-1} \frac{\partial \yfmean}{\partial\hpmean^\Transp}  - \sigma^{-2}_{\star|y} \dmean \dmean^\Transp\right)^{\mathrlap{-1}} \dmean,
		\end{aligned}
	\end{equation}
	where the last equality follows from the matrix inversion lemma.
	Using the properties of the block-inverse of $\yfcov$, we show that the inner parenthesis equals $\frac{\partial \m^\Transp}{\partial \hpmean} (\K + \sigma^2 \mbf{I} )^{-1} \frac{\partial \m}{\partial \hpmean^\Transp}$ in Appendix~\ref{sec:appendix:proofofequality}.
\end{proof}

\textbf{Remark:}
Result~\ref{res:hcrb} is based on the framework in \textcite{Rockah&Schultheiss1987_array}, se also \textcite{vanTrees2013_detection}.
This assumes that the bias of the learning method $\what{\hp}$ is zero.
In Appendix~\ref{sec:appendixB} we provide an alternative proof of Result~\ref{res:hcrb} that relaxes this assumption.


\def \mytmp {\value{exmpcnt}}
\setcounter{exmpcnt}{0}
\begin{exmp}{(cont'd)}
To illustrate the difference between \eqref{eq:BCRB} and \eqref{eq:HCRB}, consider Figure~\ref{fig:basicexample_hcrb}.
It shows the same realization of $f(x)$ as in Figure~\ref{fig:basicexample}, along with the predicted values $\what{f}(x)$.
The error bars are now obtained from \eqref{eq:HCRB} which takes into account that $\hp$ has to be learned from the data.
These bars clearly quantify the errors more accurately and contain the realization $f(x)$, in contrast to the standard approach. 
\begin{figure}[!ht]
	\centering  
	\pgfplotsset{
		every axis/.append style={
			legend style={
				font=\small,
				row sep=-0pt,
				draw=none,
				axis line on top
				}
			},
		every axis post/.append style={
			legend style={
				draw=none,
				fill=none
				}
			},
		every axis plot post/.append style={
		  very thick,
		  every mark/.append style={
			scale=2
			}
		},
		every tick label/.append style={
			font=\footnotesize
			}
	}
	\setlength{\figureheight}{.40\columnwidth}
	\setlength{\figurewidth}{.80\columnwidth}
	%
	\ifarxiv
		\includegraphics{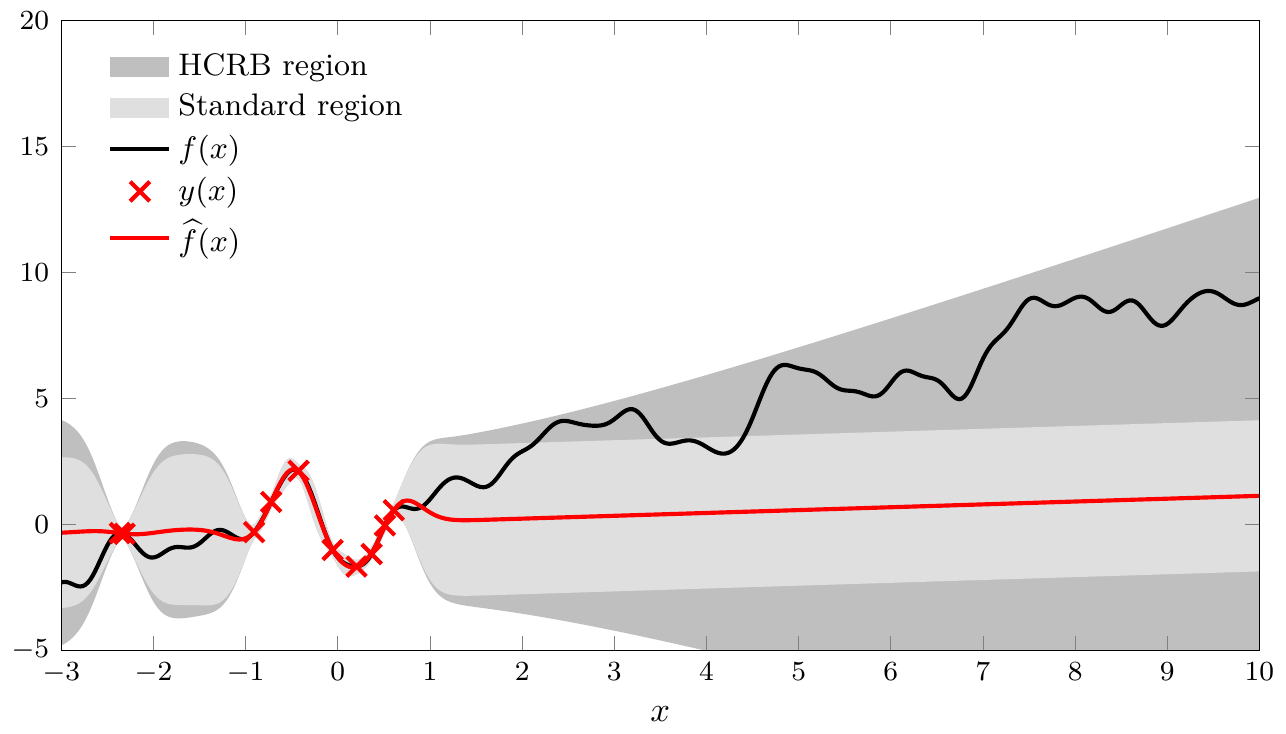}
	\else
		\tikzsetnextfilename{linear_mean_sq_exp_cov_part1_crlb}
		\input{fig/linear_mean_sq_exp_cov_part1_crlb.pgf}
	\fi

	\caption{
		Predictions of $f(x)$ using hyperparameters that have been learned from data.
		The dark shaded error bars are regions corresponding to $\what{f}(x)\pm 3 \sqrt{\HCRB}$.
	}
	\label{fig:basicexample_hcrb}
\end{figure}
\end{exmp}
\setcounter{exmpcnt}{\mytmp}


\def \mytmp {\value{exmpcnt}}
\begin{exmp}{(Time series prediction)}
Here we consider a temporal process with an unknown linear trend and periodicity (per) modeled by mean function $$m_{\alpha}(x) = \alpha_1 + \alpha_2 x,$$ covariance kernel $$k^{\text{per}}_{\beta}(x,x^\prime) = \beta^2_1 \exp\left(-\frac{2}{\beta_2^2}\sin^2\left(\frac{\pi}{\beta_3} \|x-x^\prime\|\right) + \frac{1}{\beta^2_4}\|x-x^\prime\|^2\right)$$	 and unit noise level.
In Figure~\ref{fig:periodicexample} we show a single realization of the process together with the prediction error bars computed using both the predictive variance and the HCRB.
As can be seen, $f(x)$ falls outside of the credibility region provided by the standard method.
\begin{figure}[!ht]
	\centering  
	\pgfplotsset{
	every axis/.append style={
		legend style={
			row sep=-0pt,
			draw=none
			},
		axis line on top
		},
	every axis post/.append style={
		legend style={
			font=\small,
			draw=none,
			fill=none
		},
		xmax=6,
		xmin=-3
	},
	every axis plot post/.append style={
		very thick,
		every mark/.append style={
			scale=2
			}
		},
	every tick label/.append style={
		font=\footnotesize
		},
	}
	\setlength{\figureheight}{.40\columnwidth}
	\setlength{\figurewidth}{.80\columnwidth}
	%
	\ifarxiv
		\includegraphics{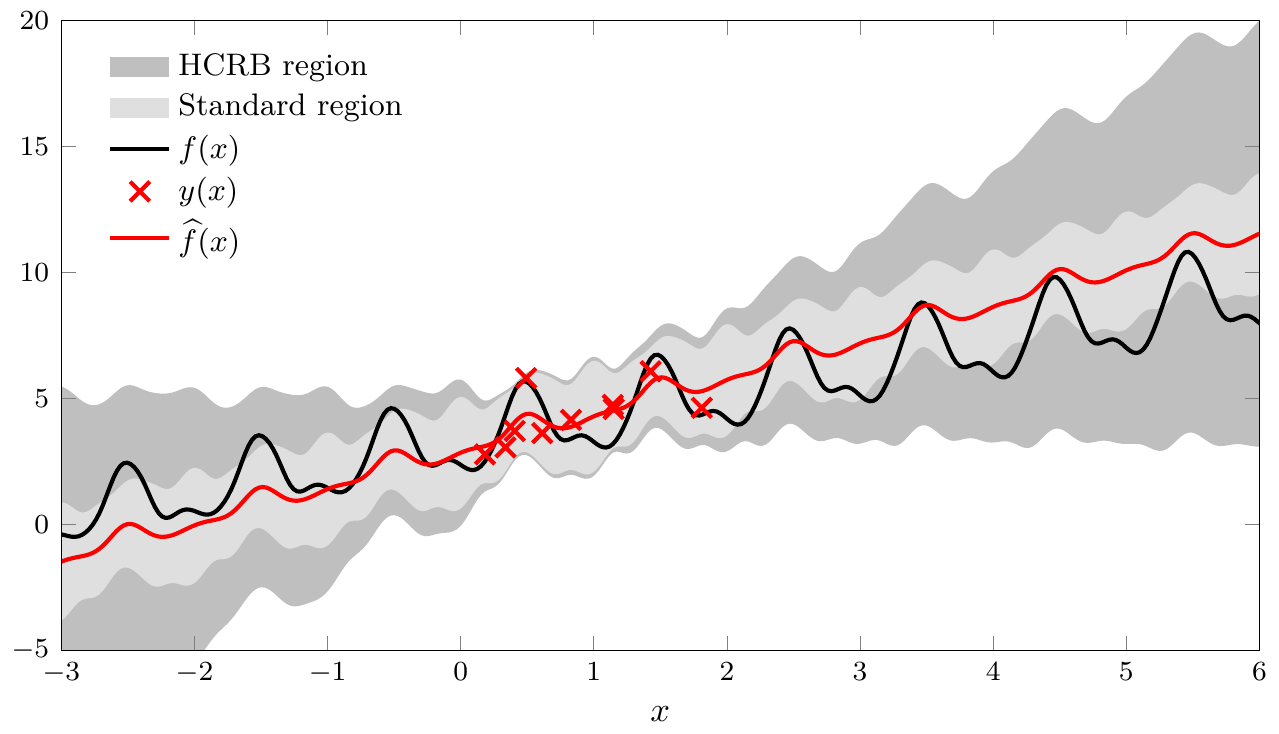}
	\else
		\tikzsetnextfilename{affine_mean_periodic_cov_part2_crlb}
		\input{fig/affine_mean_periodic_cov_part2_crlb.pgf}
	\fi

	\caption{
		Predictions of $f(x)$ using hyperparameters that have been learned from data.
		The dark shaded error bars are regions corresponding to $\what{f}(x)\pm 3 \sqrt{\HCRB}$.
	}
    \label{fig:periodicexample}
\end{figure}
\end{exmp}
\setcounter{exmpcnt}{\mytmp}

\section{Examples}
We illustrate the $\HCRB$ and its practical utility by means of several examples.
For the sake of visualization, we have considered problems with one-dimensional inputs, but the $\HCRB$ is of course valid for any dimension of $\mathcal{X}$.
The first set of examples use synthetically generated datasets in order to assess the accuracy of the error bound.
The final example uses real $\text{CO}_2$ concentration data.
We used the maximum likelihood approach to learn $\hp$ in all examples but alternative methods, such as cross-validation, could be considered as well.

\subsection{Synthetic data}
First, we consider a process $f(x)$ with the popular squared-exponential covariance (SE) function
\begin{equation}
	k^{\text{SE}}_{\beta}(x,x') = \beta^2_1 \exp\left(-\frac{1}{2\beta_2^2}\|x-x^\prime\|^2\right),
	\label{eq:SE}
\end{equation}
where we assume both the signal variance $\beta_1^2$ and length scale $\beta_2$ to be unknown.
As mean function, we assign the most basic model, a constant mean $m_{\alpha}(x) = \alpha$, where $\alpha$ is unknown. 
The unknown hyperparameters generating the data are denoted
\begin{equation}
	\hp_0 = \bbm \alpha_0 \\ \hpcov_0 \\ \sigma^2_0 \ebm,\text{ where }
	\begin{cases}
		\alpha_0 = 20, \\
		\hpcov_0 = \bbm 2 & 0.8 \ebm^{\Transp}, \\
		\sigma^2_0 = 2^2.
	\end{cases}
	\label{eq:SE_theta_0}
\end{equation}
Figure~\ref{fig:meanConst_covSEard_1000_iterations} shows the empirical MSE of \eqref{eq:predictor} after learning the hyperparameters from $N=25$ observations (obtained from $10^3$ Monte Carlo iterations), where $\what{f}(\hpmean,\hpcov,\sigma^2)$ denotes evaluating the predictor using mean parameter $\hpmean$, covariance parameter $\hpcov$ and noise level $\sigma^2$.
We compare this error with the theoretical bounds given by \eqref{eq:BCRB} and \eqref{eq:HCRB}.
We see that the bound $\HCRB(\hp_0)$ is tight as expected in this example and that $\sigma^2_{\star|y}(\hp_0)$ systematically underestimates the errors.
Similarly, when using estimated bounds by inserting the learned hyperparameters $\what{\hp}$ into \eqref{eq:BCRB} and \eqref{eq:HCRB} the gap between $\sigma^2_{\star|y}(\what{\hp})$ and $\HCRB(\what{\hp})$ not extreme, but still present, with the latter giving a better representation of the true error than the estimated predictive variance.

\begin{figure}[!htb]
	\centering  
	\pgfplotsset{
		every axis/.append style={
			legend style={
				font=\small,
				row sep=1pt
				},
			axis line on top
			},
		every axis post/.append style={
			ymin=0,
			ymax=10,
			xmin=-8,
			xmax=8,
			legend style={
				draw=none,
				fill=white,
				at={(0.60,0.97)}
			},
		},
        every axis plot post/.append style={
			mark size=3pt,
			very thick,
			},
        every tick label/.append style={
			font=\footnotesize
			}
      }
    \setlength{\figureheight}{.40\columnwidth}
    \setlength{\figurewidth}{.80\columnwidth}
	%
	\ifarxiv
		\includegraphics{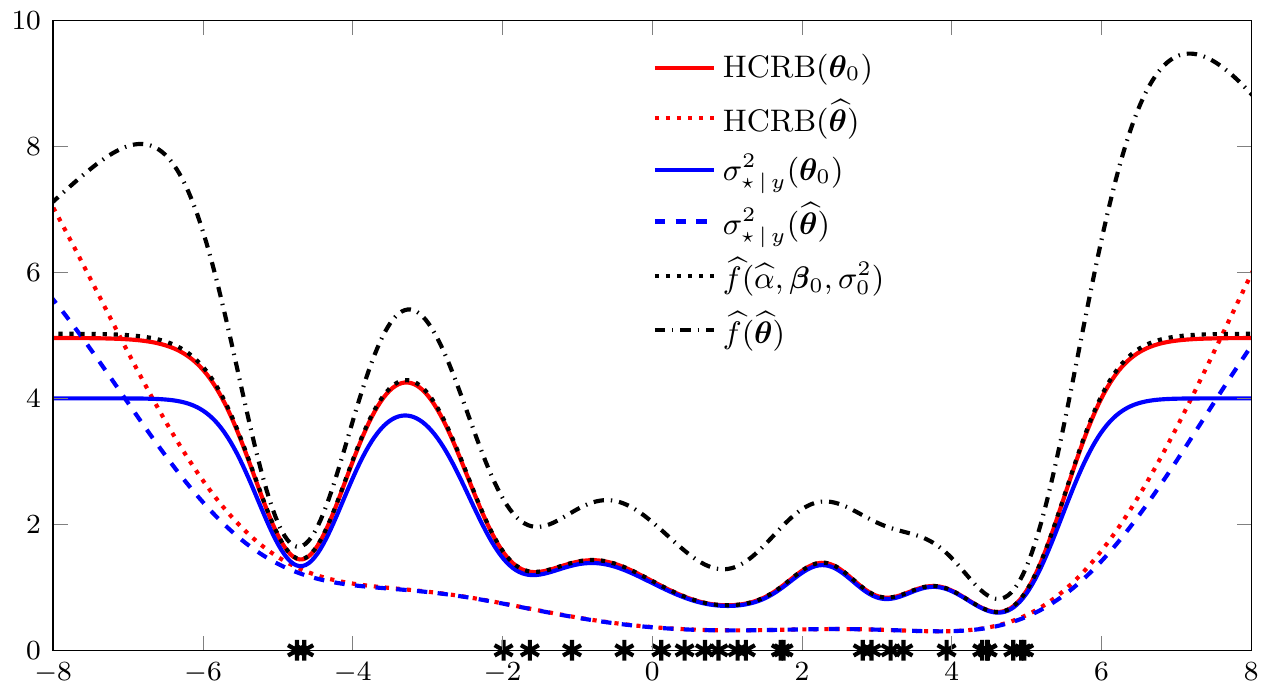}
	\else
		\tikzsetnextfilename{meanConst_covSEard_1000_iterations}
		\input{fig/meanConst_covSEard_1000_iterations.pgf}
	\fi

	\caption{
		MSE of predictors along $x \in \mathcal{X}$ using learned hyperparameters and the bounds \eqref{eq:BCRB} and \eqref{eq:HCRB}.
		The red curves show the true and estimated HCRB, based on $\hp_0$ and $\what{\hp}$, respectively.
		In blue, we show predictive variances $\sigma^2_{\star\mid y}$ corresponding to $\hp_0$ and $\what{\hp}$, respectively.
		The black dots indicate the input sample locations $x$.
		$f(x)$ has a constant mean function, $m_{\alpha}(x) = \alpha$ and a squared exponential covariance function \eqref{eq:SE} with hyperparameters as in \eqref{eq:SE_theta_0}.
	}
	\label{fig:meanConst_covSEard_1000_iterations}
\end{figure}

Another simple but common mean function is the linear mean $m_{\alpha}(x) = \alpha x$.
Again, using a process with the squared exponential covariance function,\eqref{eq:SE}, with hyperparameters $\hpcov_0$ and noise level $\sigma^2_0$ as in \eqref{eq:SE_theta_0} but with a linear mean function with $\alpha_0 = 2$, we evaluate the empirical MSE (obtained by $10^3$ Monte Carlo samples) and the theoretical bounds in Figure~\ref{fig:meanLinear_covSEard_1000_iterations}.
The hyperparameters were learned using $N = 25$ observations. The gap
between the bounds become even more pronounced in predictions outside
the sampled region. 
\begin{figure}[!htb]
	\centering  
	\pgfplotsset{
		every axis/.append style={
			legend style={
				font=\small,
				row sep=1pt
				},
			axis line on top
			},
		every axis post/.append style={
			ymin=0,
			ymax=11,
			xmin=-8,
			xmax=8,
			legend style={
				draw=none,
				fill=white,
				at={(0.60,0.97)}
			},
		},
        every axis plot post/.append style={
			mark size=3pt,
			very thick,
			},
        every tick label/.append style={
			font=\footnotesize
			}
      }
  \setlength{\figureheight}{.40\columnwidth}
  \setlength{\figurewidth}{.80\columnwidth}
	\ifarxiv
		\includegraphics{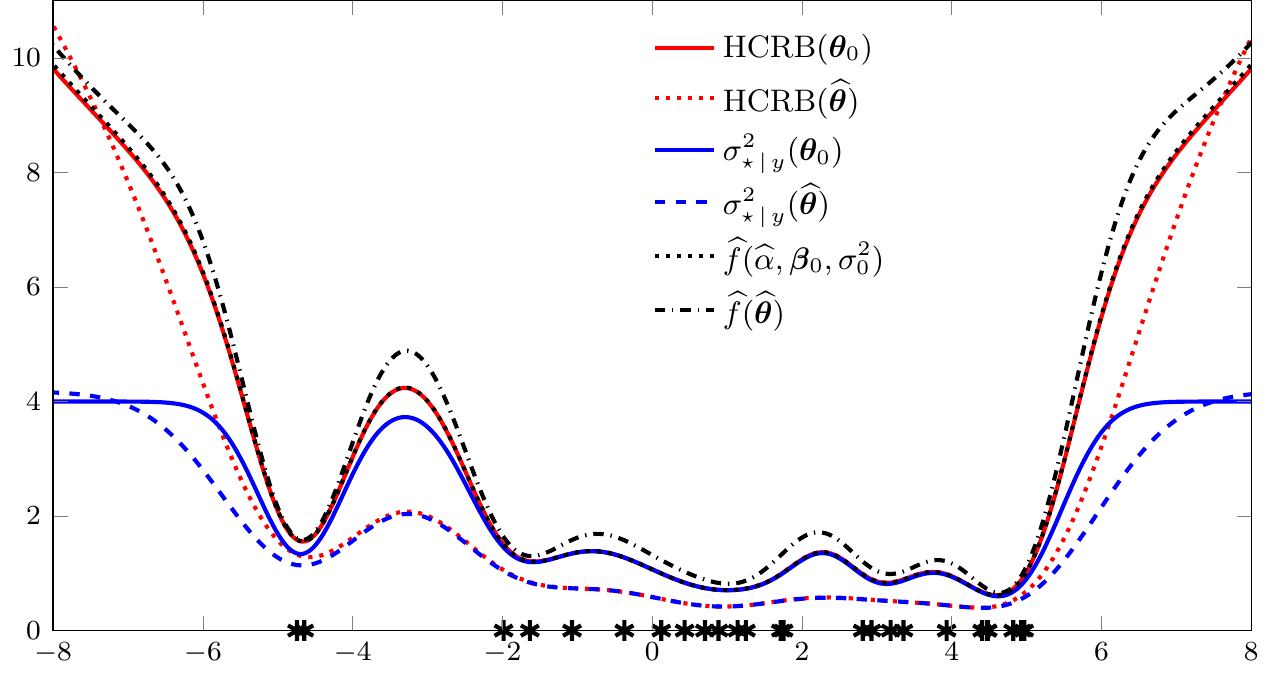}
	\else
		\tikzsetnextfilename{meanLinear_covSEard_1000_iterations}
		\input{fig/meanLinear_covSEard_1000_iterations.pgf}
	\fi

	\caption{
		MSE of predictors using learned hyperparameters and the bounds \eqref{eq:BCRB} and \eqref{eq:HCRB}.
		Here  $f(x)$ has a linear mean function, $m_{\alpha}(x) = \alpha x$ and a squared exponential covariance function \eqref{eq:SE} with $\alpha_0 = 2$, and $\hpcov_0$ and $\sigma_0^2$ as in \eqref{eq:SE_theta_0}.
	}
	\label{fig:meanLinear_covSEard_1000_iterations}
\end{figure}

Next, we consider an example inspired by frequency estimation in colored noise, which is a challenging problem.
We model a process $f(x)$ using a sinusoid mean function
\begin{equation}
	m_{\alpha}(x) = \alpha_1 \sin\left( \alpha_2 x + \alpha_3 \right).
	\label{eq:meanSinP}
\end{equation}
with unknown amplitude, frequency and phase, and a squared exponential covariance function \eqref{eq:SE}.
Here, we let $\hpmean_0 = \bbm 3 & 2 & \pi/4 \ebm^{\Transp}$, $\hpcov_0 = \bbm 0.5 & 3\ebm^{\Transp}$ and $\sigma^2_0 = 0.5^2$.
The process was sampled at $25$ non-uniformly spaced input points, cf. Figure~\ref{fig:meanSinP_covSEard_1000_iterations}.
Note how the conditional variance $\sigma^2_{\star\mid y}$ severely underestimates the uncertainty in the predictions between $-5$ and $-3$, but how the HCRB, even when estimated from data, provides a much more accurate bound.

\begin{figure}[!htb]
	\centering  
	\pgfplotsset{
		every axis/.append style={
			legend style={
				font=\small,
				row sep=1pt
				},
			axis line on top},
		every axis post/.append style={
			xmin=-8,
			xmax=8,
			legend style={
				draw=none,
				fill=white,
				at={(0.60,0.97)}
			},
		},
        every axis plot post/.append style={
			mark size=3pt,
			very thick,
			},
        every tick label/.append style={
			font=\footnotesize
			}
      }
    \setlength{\figureheight}{.40\columnwidth}
    \setlength{\figurewidth}{.80\columnwidth}
	%
	\ifarxiv
		\includegraphics{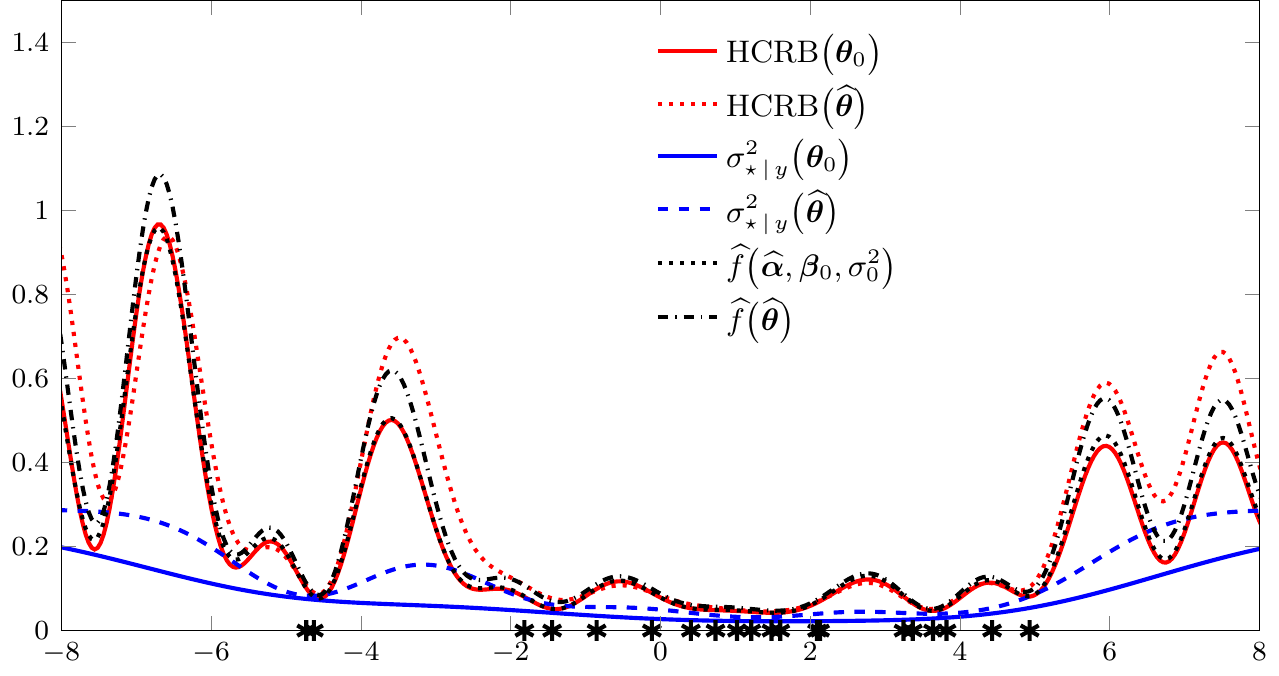}
	\else
		\tikzsetnextfilename{meanSinP_covSEard_1000_iterations}
		\input{fig/meanSinP_covSEard_1000_iterations.pgf}
	\fi

	\caption{%
		Empirical MSE and theoretical and estimated lower bound on MSE for a process $f(x)$ with a sinusoidal mean function $m_{\alpha}(x) = \alpha_1 \sin \left(\alpha_2 x + \alpha_3\right)$ and squared exponential covariance function \eqref{eq:SE}, with $\hpmean_0 = \protect\bbm 3 & 2 & \frac{\pi}{4} \protect\ebm^{\Transp}$, $\hpcov_0 = \protect\bbm 0.5^2 & 3 \protect\ebm^{\Transp}$ and $\sigma^2_0 = .5^2$.
	}
	\label{fig:meanSinP_covSEard_1000_iterations}
\end{figure}

\subsection{Marginalizing the mean parameters}

For the special case in which the mean function is linear in the parameters, that is, $m_\alpha(\x) = \hpmean^\Transp \mbf{u}(\x)$, it is possible to consider an alternative parameterization: A Gaussian hyperprior $\hpmean \sim \mathcal{N}(\0,\text{diag}(\wtilde{\hpcov}))$ can be assigned with positive covariance parameters $\wtilde{\hpcov}$.
By marginalizing out $\hpmean$ from $f(\x)$, we then obtain an additional term to the covariance function $k_{\wtilde{\beta}}(\x, \x') =
\mbf{u}^\Transp(\x)\text{diag}(\wtilde{\hpcov}) \mbf{u}(\x')$ where $\wtilde{\hpcov}$ is augmented to the hyperparameters.
Correspondly, the mean function becomes $m(\x) \equiv \0$, which is a common assumption in the Gaussian process literature \parencite{Rasmussen&Williams2006_gaussian}.
This model parameterization will therefore have an alternative predictive variance $\sigma^2_{\star | y}$ that captures the uncertainty of the linear mean parameters.

To study the effect of this alternative parameterization on the bounds and prediction, we consider a linear trend $m_\alpha(x) = \alpha_1 + \alpha_2 x$ along with $k_{\beta}(x,x^\prime) = k_{\beta^{\text{SE}}}^{\text{SE}}(x,x^\prime)$ as given in \eqref{eq:SE}.
The data was generated according to this model and the $\HCRB$ is evaluated in Figure~\ref{fig:meanAffine_vs_covAffine_all_5000_iterations_hcrb}.
The marginalized model is here $m(x) \equiv 0$ and $k_{\beta^{\text{SE}}}^{\text{SE}}(x,x^\prime) + k_{\beta^{\text{aff}}}(x,x^\prime)$, where $$k_{\beta^{\text{aff}}}(x,x^\prime) = \beta_1 + \beta_2 x x^\prime$$ corresponds to the unknown linear mean function.
The predictive variance of this model was evaluated learning its hyperparameters using data from the original model and inserting them into $\sigma^2_{\star | y}$.

For the special case in which only $\hpcov^{\text{aff}}$ is learned, the correspondence between $\sigma^2_{\star | y}$ and $\HCRB$ is striking in Figure~\ref{fig:meanAffine_vs_covAffine_all_5000_iterations_hcrb}.
In this case, also the the predictors, using the original and marginalized models, respectively, perform nearly identically.
When all hyperparameters are learned in the original and marginal models, respectively, the empirical $\sigma^2_{\star | y}$ turns out to be more accurate than the empirical $\HCRB$ in the extremes.
The results suggest that for linear mean functions there is a potential advantage in using the marginalized model to assess the prediction accuracy.
However, in this case we also note that the performance of the predictor based on the marginalized model is degraded in comparison to
that based on the original model.

\begin{figure*}[!ht]
	\centering  
	\pgfplotsset{
		every axis/.append style={
			legend style={
				font=\small,
				row sep=1pt
				},
			axis line on top
			},
		every axis post/.append style={
			xmin=-8,
			xmax=8,
			ymax=11,
			legend style={
				draw=none,
				fill=white,
				at={(0.60,0.97)}
				},
			},
        every tick label/.append style={
			font=\footnotesize
			},
		every axis plot post/.append style={
			very thick,
			mark size={3pt},
			}
      }
	\setlength{\figureheight}{.40\textwidth}
  \setlength{\figurewidth}{.80\textwidth}
  %
	\ifarxiv
		\includegraphics{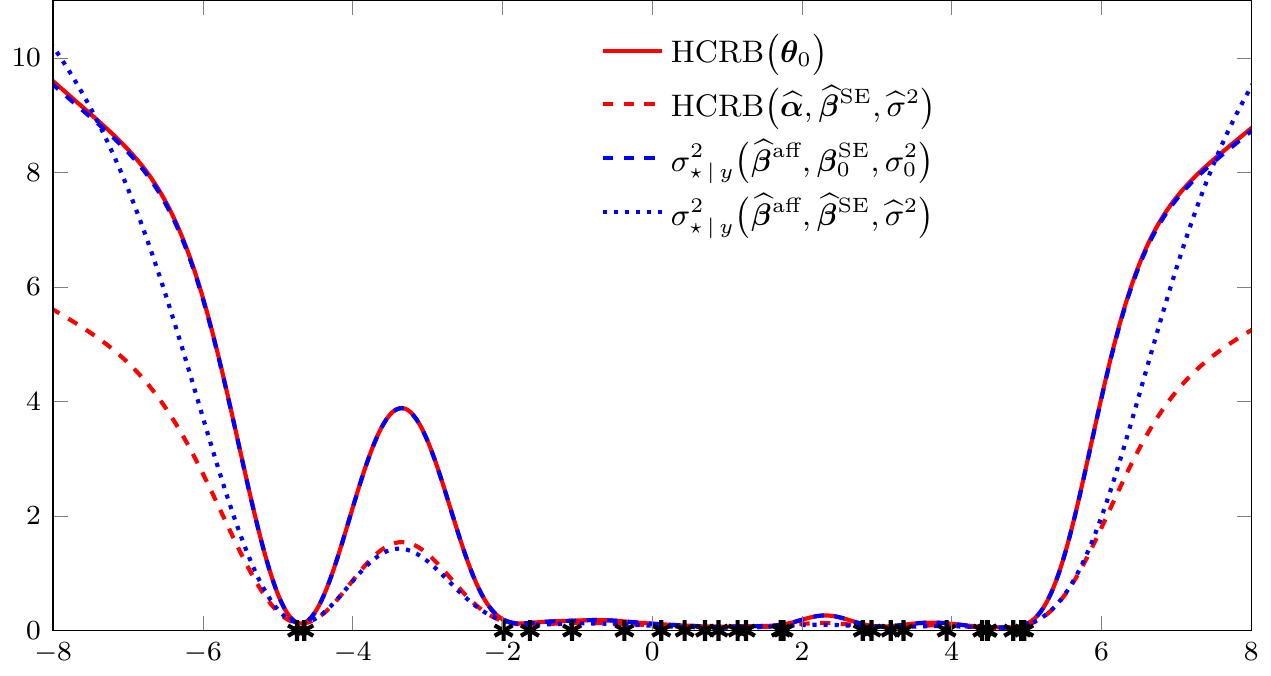}
	\else
		\tikzsetnextfilename{meanAffine_vs_covAffine_all_5000_iterations_hcrb}
		\input{fig/meanAffine_vs_covAffine_all_5000_iterations_hcrb.pgf}
	\fi

	\\\vspace{2em}
	%
	\ifarxiv
		\includegraphics{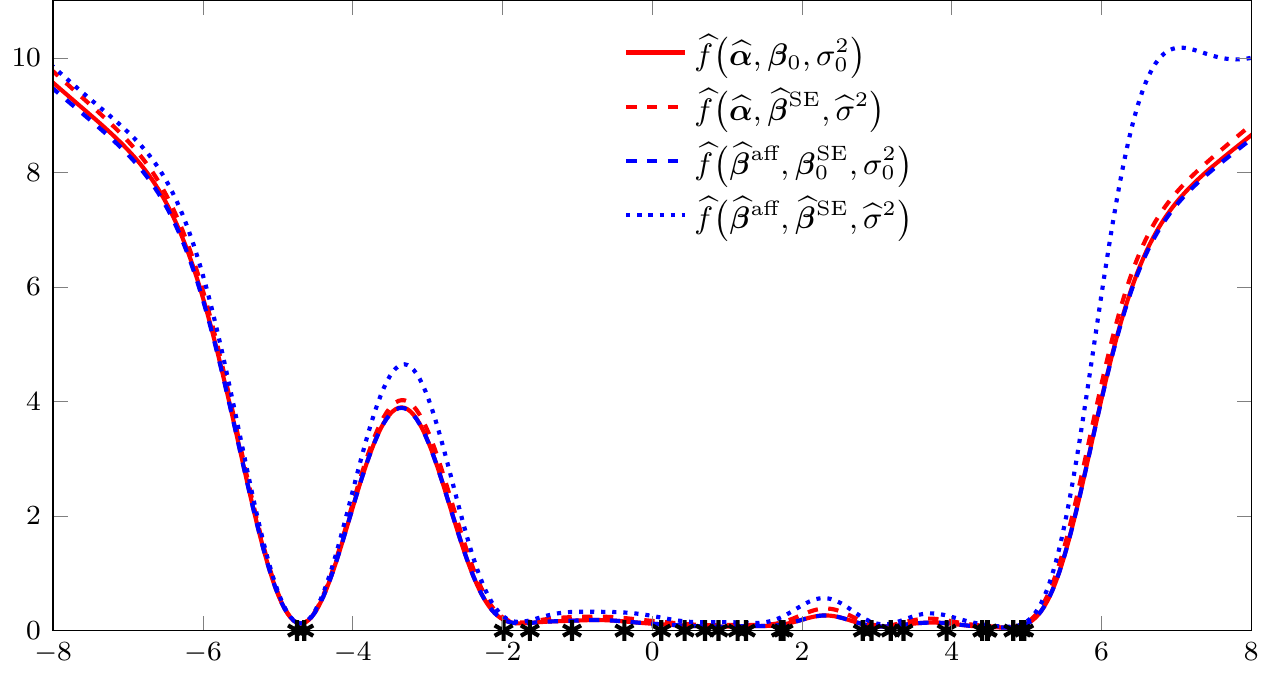}
	\else
		\tikzsetnextfilename{meanAffine_vs_covAffine_all_5000_iterations_mse}
		\input{fig/meanAffine_vs_covAffine_all_5000_iterations_mse.pgf}
	\fi

	\caption{
		A comparison between bounds and MSE for an original and marginalized data model.
		Top: The corresponding HCRB and predictive variance.
		Bottom: MSE of corresponding predictors.}
	\label{fig:meanAffine_vs_covAffine_all_5000_iterations_hcrb}
\end{figure*}

\subsection{\texorpdfstring{CO$_2$}{CO2} concentration data}

With the previous examples in mind, we now consider real $\text{CO}_2$ concentration data\footnote{\texttt{ftp://ftp.cmdl.noaa.gov/ccg/co2/trends/co2\_mm\_\allowbreak{}mlo.txt}} analyzed in \textcite{Rasmussen&Williams2006_gaussian}.
The data exhibits a trend as well as periodicities.
These features can be modeled using the mean and covariance functions considered in the previous example.
In addition, to capture smooth variations as well as erratic patterns, we consider using a squared-exponential kernel $k^{\text{SE}}(x,x^\prime)$ and a rational quadratic (RQ) kernel $k_{\beta}^{\text{RQ}}(x,x^\prime) = \beta^2_1 \left( 1 + \frac{1}{2\beta_2 \beta^2_3}\|x-x^\prime\|^2\right)^{-\beta_3}$.
The final covariance function can be written as: 
\begin{equation*}
	k_{\beta}(x,x^\prime) = k_{\beta}^{\text{SE}}(x,x^\prime) + k^{\text{per}}_{\beta}(x,x^\prime) + k_{\beta}^{\text{RQ}}(x,x^\prime).
\end{equation*}
In this example, the hyperparameters are learned using monthly data from the years 1995 to 2003.
The prediction error bars using the predictive variance and HCRB are plotted in Figure~\ref{fig:co2_3std}.
Using validation data from 2004 to March 2016 we assess the error bars.
As can be seen several data points fall outside of standard approach fall outside of the 99.7\% credibility region but are contained in the $\HCRB$ region.
\begin{figure*}[!ht]
	\centering  
	\pgfplotsset{
		every axis/.append style={legend style={font=\small,row sep=1pt},axis line on top},
		every axis post/.append style={
			ymin=355,
			ymax=420,
			xmin=1995,
			xmax=2020,
			xtick={1995,2000,...,2020},
			grid,
			grid style={dotted,black!50!white},
			legend style={draw=none,fill=white},
			legend image post style={mark size=2pt},
			/pgf/number format/.cd,
			use comma,
			1000 sep={}
		},
        every axis plot post/.append style={
          mark size=0.8pt},
        every tick label/.append style={font=\footnotesize}
      }
    \setlength{\figureheight}{.40\textwidth}
    \setlength{\figurewidth}{.80\textwidth}
  %
	\ifarxiv
		\includegraphics{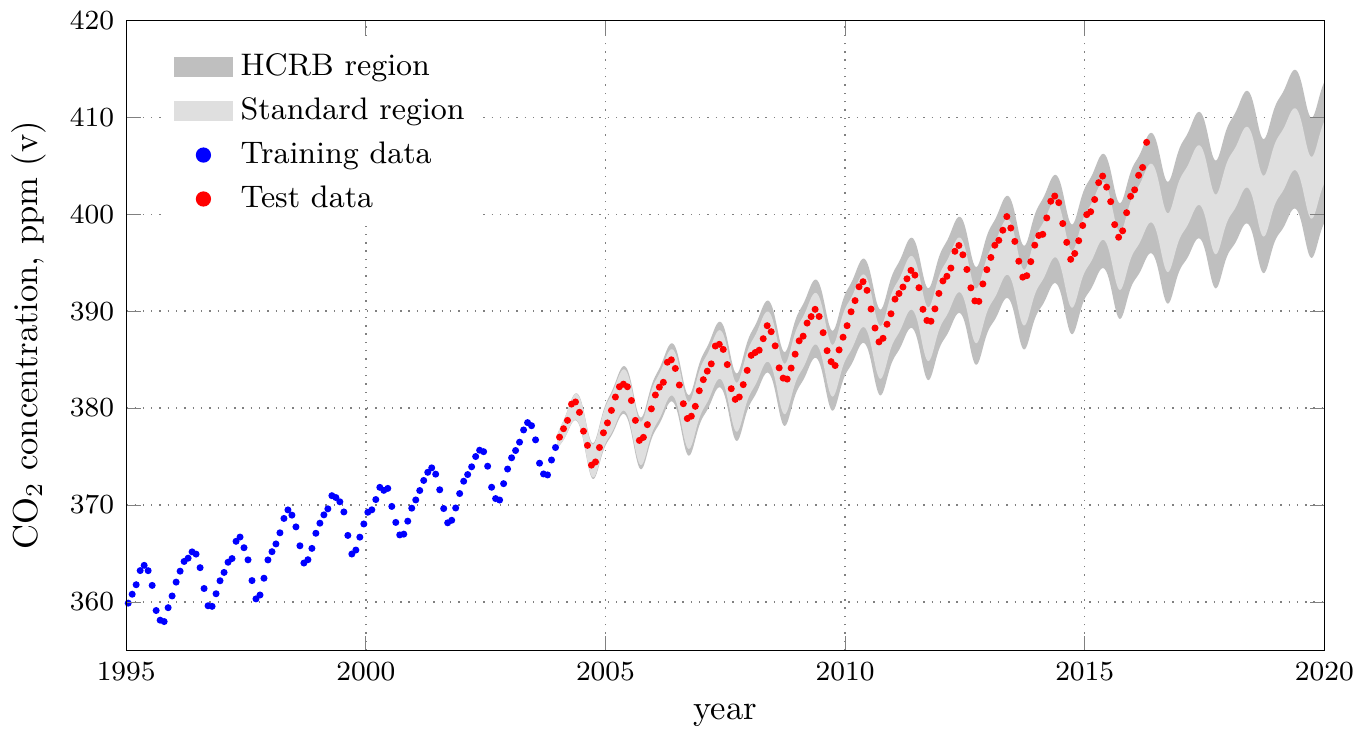}
	\else
		\tikzsetnextfilename{co2_3std}
		\input{fig/co2_3std.pgf}
	\fi

	\caption{
		Monthly average atmospheric CO$_2$ concentration measured at Mauna Loa.
		GP model fit on data until December 2003.
		Error bars based on $\what{f}(x) \pm 3\sigma_{\star|y}$ and $\what{f}(x) \pm 3\sqrt{\HCRB}$.
	}
	\label{fig:co2_3std}
\end{figure*}

\section{Discussion}
We used the Hybrid Cram\'er-Rao Bound as a tool to analyze the prediction performance of Gaussian process regression after learning.
When comparing the new bound with the commonly used predictive variance we showed that the latter will systematically underestimate the minimum MSE, even for the simplest datasets with unknown constant mean.
This leads to incorrect prediction error bars.
The underestimation gap arises from uncertainty of the hyperparameters and we provide an explicit and general characterization of it.
The resulting HCRB is a simple closed-form expression and computationally cheap to implement.

In the examples we showed that the HCRB provides a tighter lower bound of the MSE for the standard predictor than the nominal predictive variance.
The $\HCRB$ is easily computed using the quantities in the predictor itself and provides more accurate error bars, even when using estimated hyperparameters.
In future work, we will investigate the accuracy of the estimated $\HCRB$ further.
For the special case of linear mean functions, the results indicate a possible advantage of using an
alternative marginalized model and assessing its corresponding BCRB
using learned model parameters.

\subsubsection*{Acknowledgments}
The authors would like to thank Dr. Marc Deisenroth and Prof. Carl E. Rasmussen for fruitful discussions.

This research is financially supported by the Swedish Foundation for Strategic Research (SSF) via the project \emph{ASSEMBLE} (Contract number: RIT15-0012) .
The work was also supported by the Swedish research Council (VR) via the projects \emph{Probabilistic modeling of dynamical systems} (Contract number: 621-2013-5524) and (Contract number: 621-2014-5874).

\appendix
\section{Alternative derivation of the bound}\label{sec:appendixB}
Unlike \textcite{Rockah&Schultheiss1987_array,vanTrees2013_detection}, we
will here prove Result~\ref{res:hcrb} assuming only that the bias of
the estimator with respect to $\check{f}_\star$ is invariant to
$\hp$. That is, 
\begin{equation*}
b(\hp) \triangleq \E_{y}\left[  \check{f}_\star(\hp) - \what{f}_\star
  \right] \equiv b,
\end{equation*}
where $b$ is a constant.

We begin by decomposing the MSE of an estimator $\what{f}_\star$:
\begin{equation}\
\begin{split}
\MSE{\what{f}_\star} &= \E[|f_\star - \what{f}_\star |^2] \\
&= \E_{y}\left[ \E_{f|y}\left[|f_\star - \check{f}_\star + \check{f}_\star - \what{f}_\star |^2\right] \right] \\
&= \sigma^2_{\star|y} + \E_{y}\left[ | \check{f}_\star - \what{f}_\star  |^2 \right],
\end{split}
\label{eq:MSEdecomposition}
\end{equation}
where $\check{f}_\star = \check{f}_\star(\hp)$ is the conditional mean 
\eqref{eq:predictor} of $f_\star$. Since the first term in \eqref{eq:MSEdecomposition}, $\sigma^2_{\star|y}$, is independent of
the estimator we will focus on finding a lower bound for the second term.

For notational simplicity, define the score function of the training
data pdf as:
$$\ivec = \frac{\pder}{\pder \hp} \ln p(\y | \hp ).$$
Then the correlation between the score function and the estimation
error is
\begin{equation*}
\begin{split}
\wtilde{\mbf{g}} &= \E_{y}\left[ \ivec (  \check{f}_\star - \what{f}_\star
  )\right] \\ &= \int \left[\frac{\pder}{\pder \hp} p(\y|\hp) \right] (  \check{f}_\star - \what{f}_\star
  )  \: d \y \\
&= \int \frac{\pder}{\pder \hp} \left[ p(\y|\hp)  (  \check{f}_\star - \what{f}_\star
  )\right]  - p(\y|\hp) \left[ \frac{\pder}{\pder \hp} (  \check{f}_\star - \what{f}_\star
  ) \right] \: d \y  \\
&= 0 -\E_{y} \left[ \frac{\pder}{\pder \hp}   \check{f}_\star \right] \\
&= -\begin{bmatrix} \mbf{g}^{\Transp} & \0 & 0 \end{bmatrix}^{\Transp}.
\end{split}
\end{equation*}
The first set of zeros follows from
\begin{equation*}
\begin{split}
\E_{y}\left[ \frac{\partial}{\partial \hpcov} \check{f}_\star  \right]
= \left(
\frac{\partial}{\partial \hpcov} \linpred^\Transp \right) \E_{y}\left[  
 (\y - \m) \right] = \0.
\end{split}
\end{equation*}
The final zero follows in a similar manner.

The Fisher information matrix is given by Slepian-Bangs formula:
\begin{equation*}
\begin{split}
\wtilde{\FIM} &\triangleq \E_{y} \left[ \frac{\pder}{\pder \hp} \ln p(\y | \hp ) \frac{\pder}{\pder \hp} \ln p(\y | \hp )^\Transp \right] = \begin{bmatrix} \mbf{M} &  \0 & \0 \\ \0  & * & * \\
 \0  & * & * \end{bmatrix}
\end{split}
\end{equation*}
The zeros therefore follow from the properties of the Gaussian distribution.
We now form the product $\wtilde{\mbf{g}}^\Transp \wtilde{\FIM}^{-1}
  \ivec$ and the non-negative quadratic function
\begin{equation*}
\begin{split}
0 &\leq \E_{y}\left[ | (\check{f}_\star - \what{f}_\star) - \wtilde{\mbf{g}}^\Transp \wtilde{\FIM}^{-1}
  \ivec |^2 \right] \\
&= \E_{y}\left[ | \check{f}_\star - \what{f}_\star  |^2 \right] +
\wtilde{\mbf{g}}^\Transp \wtilde{\FIM}^{-1}\wtilde{\mbf{g}} - 2
\wtilde{\mbf{g}}^\Transp \wtilde{\FIM}^{-1} \E_{y}\left[ \ivec ( \check{f}_\star - \what{f}_\star
  )\right] \\
 &= \E_{y}\left[ | \check{f}_\star - \what{f}_\star  |^2 \right]-  \wtilde{\mbf{g}}^\Transp \wtilde{\FIM}^{-1}\wtilde{\mbf{g}}.
\end{split}
\end{equation*}
It follows that
\begin{equation*}
\begin{split}
\E_{y}\left[ | \check{f}_\star - \what{f}_\star  |^2 \right] \: \geq \: \wtilde{\mbf{g}}^\Transp
\wtilde{\FIM}^{-1}\wtilde{\mbf{g}}  
\end{split}
\end{equation*}
Thus \eqref{eq:MSEdecomposition} is lower bounded by
\begin{equation*}
\begin{split}
\MSE{\what{f}_\star} \geq \sigma^2_{\star|y}  + \mbf{g}^\Transp
\mbf{M}^{-1} \mbf{g} ,
\end{split}
\end{equation*}
which is Result~\ref{res:hcrb}.

\section{Proof of equality}\label{sec:appendix:proofofequality}
 Recall that
 $\yfmean = \bbm \m^{\Transp} & m_{\star} \ebm^{\Transp}$, 
 $\Sigma = \bbm \K+\sigma^{2}\mbf{I} & \mbf{k}_{\star} \\ \mbf{k}_{\star}^{\Transp} & k_{\star\star}\ebm$,
 $\linpred = \left( \mbf{K} + \hpstd^{2} \mbf{I} \right)^{-1} \mbf{k}_{\star}$,
 $\dmean = \frac{\partial}{\partial \hpmean} (m_{\star} - \linpred^{\top} \m )$, 
 and $\sigma^{2}_{\star\mid y} = k_{\star\star} - \mbf{k}^{\Transp}_{\star}
 \left( \mbf{K} + \hpstd^{2} \mbf{I} \right)^{-1} \mbf{k}_{\star}$.
 Then the following holds.
 \begin{equation*}
 \frac{\partial \yfmean^{\Transp}}{\partial
     \hpmean} \yfcov^{-1} \frac{\partial \yfmean}{\partial
   \hpmean^{\Transp}}  - \sigma^{-2}_{\star|y} \dmean \dmean^{\Transp}
 =\frac{\partial
   \m^{\Transp}}{\partial \hpmean} (\K + \sigma^{2} \mbf{I} )^{-1}
 \frac{\partial \m}{\partial \hpmean^{\Transp}}
 \end{equation*}
 \begin{proof}
 \newcommand{\Sy}{\mbf{\Sigma}_{y}}
 Let $\Sy = \K + \sigma^2\mbf{I}$.
 \begin{align*}
 \frac{\partial \yfmean^{\Transp}}{\partial \hpmean} \yfcov^{-1} \frac{\partial \yfmean}{\partial \hpmean^{\Transp}}
     &= \frac{\partial \yfmean^{\Transp}}{\partial \hpmean} \bbm \Sy & \mbf{k}_{\star} \\ \mbf{k}_{\star}^{\Transp} & k_{\star\star}\ebm^{-1}\frac{\partial \yfmean}{\partial \hpmean^{\Transp}} \\
     &= \frac{\partial }{\partial \hpmean}\bbm \m^{\Transp} & m_{\star} \ebm \bbm \Sy^{-1} + \Sy^{-1}\mbf{k}_{\star}\sigma_{\star\mid y}^{-2}\mbf{k}_{\star}^{\Transp}\Sy^{-1} & -\Sy^{-1}\mbf{k}_{\star}\sigma_{\star\mid y}^{-2} \\ -\sigma_{\star\mid y}^{-2}\mbf{k}_{\star}^{\Transp}\Sy^{-1}  & \sigma_{\star\mid y}^{-2}\ebm\frac{\partial}{\partial \hpmean^{\Transp}}\bbm \m \\ m_{\star} \ebm \\
     &= \frac{\partial }{\partial \hpmean}\bbm \m^{\Transp} & m_{\star} \ebm \bbm \Sy^{-1} & 0 \\ 0 & 0\ebm\frac{\partial}{\partial \hpmean^{\Transp}}\bbm \m \\ m_{\star} \ebm +\sigma_{\star\mid y}^{-2} \frac{\partial }{\partial \hpmean}\bbm \m^{\Transp} & m_{\star} \ebm \bbm \linpred \linpred^{\Transp} & -\linpred \\ -\linpred^{\Transp} & 1\ebm\frac{\partial}{\partial \hpmean^{\Transp}}\bbm \m \\ m_{\star} \ebm \\
     &= \frac{\partial \m^{\Transp}}{\partial \hpmean} \Sy^{-1} \frac{\partial \m}{\partial \hpmean^{\Transp}} + \sigma_{\star \mid y}^{-2} \frac{\partial}{\partial \hpmean}(m_{\star}-\linpred^{\Transp} \m)\frac{\partial}{\partial \hpmean^{\Transp}}(m_{\star}-\linpred^{\Transp} \m)^{\Transp} \\
     &= \frac{\partial \m^{\Transp}}{\partial \hpmean} \left( \K + \sigma^{2}\mbf{I}\right)^{-1} \frac{\partial \m}{\partial \hpmean^{\Transp}} + \sigma_{\star \mid y}^{-2} \dmean \dmean^{\Transp}
 \end{align*}
 \end{proof}

\printbibliography

\end{document}